\documentclass{article}
\usepackage[T1]{fontenc}
\usepackage[utf8]{inputenc}
\usepackage{units}
\usepackage{amsmath}
\usepackage{amsthm}
\usepackage{graphicx}
\usepackage{esint}
\usepackage{xargs}[2008/03/08]
\usepackage[unicode=true,
 bookmarks=false,
 breaklinks=false,pdfborder={0 0 1},backref=section,colorlinks=false]
 {hyperref}

\makeatletter

\providecommand{\tabularnewline}{\\}

\theoremstyle{plain}
\newtheorem{thm}{\protect\theoremname}
  \theoremstyle{plain}
  \newtheorem{prop}[thm]{\protect\propositionname}

\usepackage[final, nonatbib]{nips_2017}

\usepackage{url}
\usepackage{amsfonts}
\usepackage{nicefrac}
\usepackage{microtype}

\title{Dual Discriminator Generative Adversarial Nets}

\author{
  Tu Dinh Nguyen, Trung Le, Hung Vu, Dinh Phung\\
  Centre for Pattern Recognition and Data Analytics\\
  Deakin University, Australia\\
  \texttt{\{tu.nguyen,trung.l,hungv,dinh.phung\}@deakin.edu.au} \\
}

\usepackage{graphicx}

\usepackage{colortbl} 
\definecolor{header_color}{rgb}{0.74,0.88,0.91}
\definecolor{even_color}{rgb}{0.9,0.9,0.9}
\definecolor{subheader_color}{rgb}{0.85,0.93,0.95}
\definecolor{childheader_color}{rgb}{1.0,0.93,0.87}

\definecolor{ccolor_best}{rgb}{1.0,0.9,0.9}
\definecolor{ccolor_wrong}{rgb}{1.0,0.85,0.85}

\usepackage{caption}

\usepackage{thmtools}

\declaretheoremstyle[%
  spaceabove=-6pt,%
  spacebelow=6pt,%
  headfont=\normalfont\itshape,%
  postheadspace=1em,%
  qed=\qedsymbol%
]{mystyle}

\@ifundefined{showcaptionsetup}{}{%
 \PassOptionsToPackage{caption=false}{subfig}}
\usepackage{subfig}
\makeatother

  \providecommand{\propositionname}{Proposition}
\providecommand{\theoremname}{Theorem}

\begin{document}

\maketitle
\begin{abstract}
We propose in this paper a novel approach to tackle the problem
of mode collapse encountered in generative adversarial network (GAN).
Our idea is intuitive but proven to be very effective, especially
in addressing some key limitations of GAN. In essence, it combines
the Kullback-Leibler (KL) and reverse KL divergences into a unified
objective function, thus it exploits the complementary statistical
properties from these divergences to effectively diversify the estimated
density in capturing multi-modes. We term our method \emph{dual discriminator
generative adversarial nets} (D2GAN) which, unlike GAN, has \emph{two}
discriminators; and together with a generator, it also has the analogy
of a minimax game, wherein a discriminator rewards high scores for
samples from data distribution whilst another discriminator, conversely,
favoring data from the generator, and the generator produces data
to fool both two discriminators. We develop theoretical analysis to
show that, given the maximal discriminators, optimizing the generator
of D2GAN reduces to minimizing both KL and reverse KL divergences
between data distribution and the distribution induced from the data
generated by the generator, hence effectively avoiding the mode collapsing
problem. We conduct extensive experiments on synthetic and real-world
large-scale datasets (MNIST, CIFAR-10, STL-10, ImageNet), where we
have made our best effort to compare our D2GAN with the latest state-of-the-art
GAN's variants in comprehensive qualitative and quantitative evaluations.
The experimental results demonstrate the competitive and superior
performance of our approach in generating good quality and diverse
samples over baselines, and the capability of our method to scale
up to ImageNet database.

\end{abstract}
\newcommand{\sidenote}[1]{\marginpar{\small \emph{\color{Medium}#1}}}

\global\long\def\se{\hat{\text{se}}}

\global\long\def\interior{\text{int}}

\global\long\def\boundary{\text{bd}}

\global\long\def\ML{\textsf{ML}}

\global\long\def\GML{\mathsf{GML}}

\global\long\def\HMM{\mathsf{HMM}}

\global\long\def\support{\text{supp}}

\global\long\def\new{\text{*}}

\global\long\def\stir{\text{Stirl}}

\global\long\def\mA{\mathcal{A}}

\global\long\def\mB{\mathcal{B}}

\global\long\def\expect{\mathbb{E}}

\global\long\def\mF{\mathcal{F}}

\global\long\def\mK{\mathcal{K}}

\global\long\def\mH{\mathcal{H}}

\global\long\def\mX{\mathcal{X}}

\global\long\def\mZ{\mathcal{Z}}

\global\long\def\mS{\mathcal{S}}

\global\long\def\Ical{\mathcal{I}}

\global\long\def\mT{\mathcal{T}}

\global\long\def\Pcal{\mathcal{P}}

\global\long\def\dist{d}

\global\long\def\HX{\entro\left(X\right)}
 \global\long\def\entropyX{\HX}

\global\long\def\HY{\entro\left(Y\right)}
 \global\long\def\entropyY{\HY}

\global\long\def\HXY{\entro\left(X,Y\right)}
 \global\long\def\entropyXY{\HXY}

\global\long\def\mutualXY{\mutual\left(X;Y\right)}
 \global\long\def\mutinfoXY{\mutualXY}

\global\long\def\given{\mid}

\global\long\def\gv{\given}

\global\long\def\goto{\rightarrow}

\global\long\def\asgoto{\stackrel{a.s.}{\longrightarrow}}

\global\long\def\pgoto{\stackrel{p}{\longrightarrow}}

\global\long\def\dgoto{\stackrel{d}{\longrightarrow}}

\global\long\def\lik{\mathcal{L}}

\global\long\def\logll{\mathit{l}}

\global\long\def\bigcdot{\raisebox{-0.5ex}{\scalebox{1.5}{\ensuremath{\cdot}}}}

\global\long\def\sig{\textrm{sig}}

\global\long\def\likelihood{\mathcal{L}}

\global\long\def\vectorize#1{\mathbf{#1}}

\global\long\def\vt#1{\mathbf{#1}}

\global\long\def\gvt#1{\boldsymbol{#1}}

\global\long\def\idp{\ \bot\negthickspace\negthickspace\bot\ }
 \global\long\def\cdp{\idp}

\global\long\def\das{}

\global\long\def\id{\mathbb{I}}

\global\long\def\idarg#1#2{\id\left\{  #1,#2\right\}  }

\global\long\def\iid{\stackrel{\text{iid}}{\sim}}

\global\long\def\bzero{\vt 0}

\global\long\def\bone{\mathbf{1}}

\global\long\def\a{\mathrm{a}}

\global\long\def\ba{\mathbf{a}}

\global\long\def\b{\mathrm{b}}

\global\long\def\bb{\mathbf{b}}

\global\long\def\B{\mathrm{B}}

\global\long\def\boldm{\boldsymbol{m}}

\global\long\def\c{\mathrm{c}}

\global\long\def\C{\mathrm{C}}

\global\long\def\d{\mathrm{d}}

\global\long\def\D{\mathrm{D}}

\global\long\def\N{\mathrm{N}}

\global\long\def\h{\mathrm{h}}

\global\long\def\H{\mathrm{H}}

\global\long\def\bH{\mathbf{H}}

\global\long\def\K{\mathrm{K}}

\global\long\def\M{\mathrm{M}}

\global\long\def\bff{\vt f}

\global\long\def\bx{\mathbf{\mathbf{x}}}

\global\long\def\bl{\boldsymbol{l}}

\global\long\def\s{\mathrm{s}}

\global\long\def\T{\mathrm{T}}

\global\long\def\bu{\mathbf{u}}

\global\long\def\v{\mathrm{v}}

\global\long\def\bv{\mathbf{v}}

\global\long\def\bo{\boldsymbol{o}}

\global\long\def\bh{\mathbf{h}}

\global\long\def\bs{\boldsymbol{s}}

\global\long\def\x{\mathrm{x}}

\global\long\def\bx{\mathbf{x}}

\global\long\def\bz{\mathbf{z}}

\global\long\def\hbz{\hat{\bz}}

\global\long\def\z{\mathrm{z}}

\global\long\def\y{\mathrm{y}}

\global\long\def\bxnew{\boldsymbol{y}}

\global\long\def\bX{\boldsymbol{X}}

\global\long\def\tbx{\tilde{\bx}}

\global\long\def\by{\boldsymbol{y}}

\global\long\def\bY{\boldsymbol{Y}}

\global\long\def\bZ{\boldsymbol{Z}}

\global\long\def\bU{\boldsymbol{U}}

\global\long\def\bn{\boldsymbol{n}}

\global\long\def\bV{\boldsymbol{V}}

\global\long\def\bI{\boldsymbol{I}}

\global\long\def\J{\mathrm{J}}

\global\long\def\bJ{\mathbf{J}}

\global\long\def\w{\mathrm{w}}

\global\long\def\bw{\vt w}

\global\long\def\bW{\mathbf{W}}

\global\long\def\balpha{\gvt{\alpha}}

\global\long\def\bdelta{\boldsymbol{\delta}}

\global\long\def\bsigma{\gvt{\sigma}}

\global\long\def\bbeta{\gvt{\beta}}

\global\long\def\bmu{\gvt{\mu}}

\global\long\def\btheta{\boldsymbol{\theta}}

\global\long\def\blambda{\boldsymbol{\lambda}}

\global\long\def\bgamma{\boldsymbol{\gamma}}

\global\long\def\bpsi{\boldsymbol{\psi}}

\global\long\def\bphi{\boldsymbol{\phi}}

\global\long\def\bpi{\boldsymbol{\pi}}

\global\long\def\bomega{\boldsymbol{\omega}}

\global\long\def\bepsilon{\boldsymbol{\epsilon}}

\global\long\def\btau{\boldsymbol{\tau}}

\global\long\def\bxi{\boldsymbol{\xi}}

\global\long\def\realset{\mathbb{R}}

\global\long\def\realn{\realset^{n}}

\global\long\def\integerset{\mathbb{Z}}

\global\long\def\natset{\integerset}

\global\long\def\integer{\integerset}

\global\long\def\natn{\natset^{n}}

\global\long\def\rational{\mathbb{Q}}

\global\long\def\rationaln{\rational^{n}}

\global\long\def\complexset{\mathbb{C}}

\global\long\def\comp{\complexset}

\global\long\def\compl#1{#1^{\text{c}}}

\global\long\def\and{\cap}

\global\long\def\compn{\comp^{n}}

\global\long\def\comb#1#2{\left({#1\atop #2}\right) }

\global\long\def\nchoosek#1#2{\left({#1\atop #2}\right)}

\global\long\def\param{\vt w}

\global\long\def\Param{\Theta}

\global\long\def\meanparam{\gvt{\mu}}

\global\long\def\Meanparam{\mathcal{M}}

\global\long\def\meanmap{\mathbf{m}}

\global\long\def\logpart{A}

\global\long\def\simplex{\Delta}

\global\long\def\simplexn{\simplex^{n}}

\global\long\def\dirproc{\text{DP}}

\global\long\def\ggproc{\text{GG}}

\global\long\def\DP{\text{DP}}

\global\long\def\ndp{\text{nDP}}

\global\long\def\hdp{\text{HDP}}

\global\long\def\gempdf{\text{GEM}}

\global\long\def\rfs{\text{RFS}}

\global\long\def\bernrfs{\text{BernoulliRFS}}

\global\long\def\poissrfs{\text{PoissonRFS}}

\global\long\def\grad{\gradient}
 \global\long\def\gradient{\nabla}

\global\long\def\partdev#1#2{\partialdev{#1}{#2}}
 \global\long\def\partialdev#1#2{\frac{\partial#1}{\partial#2}}

\global\long\def\partddev#1#2{\partialdevdev{#1}{#2}}
 \global\long\def\partialdevdev#1#2{\frac{\partial^{2}#1}{\partial#2\partial#2^{\top}}}

\global\long\def\closure{\text{cl}}

\global\long\def\cpr#1#2{\Pr\left(#1\ |\ #2\right)}

\global\long\def\var{\text{Var}}

\global\long\def\Var#1{\text{Var}\left[#1\right]}

\global\long\def\cov{\text{Cov}}

\global\long\def\Cov#1{\cov\left[ #1 \right]}

\global\long\def\COV#1#2{\underset{#2}{\cov}\left[ #1 \right]}

\global\long\def\corr{\text{Corr}}

\global\long\def\sst{\text{T}}

\global\long\def\SST{\sst}

\global\long\def\ess{\mathbb{E}}

\global\long\def\Ess#1{\ess\left[#1\right]}

\newcommandx\ESS[2][usedefault, addprefix=\global, 1=]{\underset{#2}{\ess}\left[#1\right]}

\global\long\def\fisher{\mathcal{F}}

\global\long\def\bfield{\mathcal{B}}
 \global\long\def\borel{\mathcal{B}}

\global\long\def\bernpdf{\text{Bernoulli}}

\global\long\def\betapdf{\text{Beta}}

\global\long\def\dirpdf{\text{Dir}}

\global\long\def\gammapdf{\text{Gamma}}

\global\long\def\gaussden#1#2{\text{Normal}\left(#1, #2 \right) }

\global\long\def\gauss{\mathbf{N}}

\global\long\def\gausspdf#1#2#3{\text{Normal}\left( #1 \lcabra{#2, #3}\right) }

\global\long\def\multpdf{\text{Mult}}

\global\long\def\poiss{\text{Pois}}

\global\long\def\poissonpdf{\text{Poisson}}

\global\long\def\pgpdf{\text{PG}}

\global\long\def\wshpdf{\text{Wish}}

\global\long\def\iwshpdf{\text{InvWish}}

\global\long\def\nwpdf{\text{NW}}

\global\long\def\niwpdf{\text{NIW}}

\global\long\def\studentpdf{\text{Student}}

\global\long\def\unipdf{\text{Uni}}

\global\long\def\transp#1{\transpose{#1}}
 \global\long\def\transpose#1{#1^{\mathsf{T}}}

\global\long\def\mgt{\succ}

\global\long\def\mge{\succeq}

\global\long\def\idenmat{\mathbf{I}}

\global\long\def\trace{\mathrm{tr}}

\global\long\def\argmax#1{\underset{_{#1}}{\text{argmax}} }

\global\long\def\argmin#1{\underset{_{#1}}{\text{argmin}\ } }

\global\long\def\diag{\text{diag}}

\global\long\def\norm{}

\global\long\def\spn{\text{span}}

\global\long\def\vtspace{\mathcal{V}}

\global\long\def\field{\mathcal{F}}
 \global\long\def\ffield{\mathcal{F}}

\global\long\def\inner#1#2{\left\langle #1,#2\right\rangle }
 \global\long\def\iprod#1#2{\inner{#1}{#2}}

\global\long\def\dprod#1#2{#1 \cdot#2}

\global\long\def\norm#1{\left\Vert #1\right\Vert }

\global\long\def\entro{\mathbb{H}}

\global\long\def\entropy{\mathbb{H}}

\global\long\def\Entro#1{\entro\left[#1\right]}

\global\long\def\Entropy#1{\Entro{#1}}

\global\long\def\mutinfo{\mathbb{I}}

\global\long\def\relH{\mathit{D}}

\global\long\def\reldiv#1#2{\relH\left(#1||#2\right)}

\global\long\def\KL{KL}

\global\long\def\KLdiv#1#2{\KL\left(#1\parallel#2\right)}
 \global\long\def\KLdivergence#1#2{\KL\left(#1\ \parallel\ #2\right)}

\global\long\def\crossH{\mathcal{C}}
 \global\long\def\crossentropy{\mathcal{C}}

\global\long\def\crossHxy#1#2{\crossentropy\left(#1\parallel#2\right)}

\global\long\def\breg{\text{BD}}

\global\long\def\lcabra#1{\left|#1\right.}

\global\long\def\lbra#1{\lcabra{#1}}

\global\long\def\rcabra#1{\left.#1\right|}

\global\long\def\rbra#1{\rcabra{#1}}

\global\long\def\model{\text{D2GAN}}
\vspace{-2mm}

\section{Introduction}

\vspace{-2mm}
Generative models are a subarea of research that has been rapidly
growing in recent years, and successfully applied in a wide range
of modern real-world applications (e.g., see chapter 20 in \cite{goodfellow_etal_mit16_deeplearning}).
Their common approach is to address the density estimation problem
where one aims to learn a model distribution $p_{\textrm{model }}$
that approximates the true, but \emph{unknown}, data distribution
$p_{\textrm{data}}$. Methods in this approach deal with two fundamental
problems. First, the learning behaviors and performance of generative
models depend on the choice of objective functions to train them \cite{theis2015note,huszar2015not}.
The most widely-used objective, considered the de-facto standard one,
is to follow the principle of maximum likelihood estimate that seeks
model parameters to maximize the likelihood of training data. This
is equivalent to minimizing the Kullback-Leibler (KL) divergence between
data and model distributions: $D_{\textrm{KL}}\left(p_{\textrm{data}}\Vert p_{\textrm{model }}\right)$.
It has been observed that this minimization tends to result in $p_{\textrm{model }}$
that covers multiple modes of $p_{\textrm{data}}$, but may produce
completely unseen and potentially undesirable samples \cite{theis2015note}.
By contrast, another approach is to swap the arguments and instead,
minimize: $D_{\textrm{KL}}\left(p_{\textrm{model }}\Vert p_{\textrm{data}}\right)$,
which is usually referred to as the \emph{reverse} KL divergence \cite{nowozin_etal_nips16_fgan,goodfellow_nips17_gan_tutorial,huszar2015not,theis2015note}.
It is observed that optimization towards the reverse KL divergence
criteria mimics the mode-seeking process where $p_{\textrm{model}}$
concentrates on a \emph{single} mode of $p_{\textrm{data}}$ while
ignoring other modes, known as the problem of \emph{mode collapse}.
These behaviors are well-studied in \cite{theis2015note,huszar2015not,goodfellow_nips17_gan_tutorial}.

The second problem is the choice of formulation for the density function
of $p_{\textrm{model}}$ \cite{goodfellow_etal_mit16_deeplearning}.
One might choose to define an \emph{explicit} density function, and
then straightforwardly follow maximum likelihood framework to estimate
the parameters. Another approach is to estimate the data distribution
using an \emph{implicit} density function, without the need for analytical
forms of $p_{\textrm{model}}$ (e.g., see \cite{goodfellow_nips17_gan_tutorial}
for further discussions). An idea is to borrow the principle of minimal enclosing ball \cite{benhur_etal_jmlr01_svc} to
train a generator in such a way that both training and generated data, after being mapped to the feature space, are enclosed in the same sphere \cite{le_etal_arxiv17_gen}. However, the most notably pioneered class
of this approach is the generative adversarial network (GAN) \cite{goodfellow_etal_nips14_gan},
an expressive generative model that is capable of producing sharp
and realistic images for natural scenes. Different from most generative
models that maximize data likelihood or its lower bound, GAN takes
a radical approach that simulates a game between two players: a generator
$G$ that generates data by mapping samples from a noise space to
the input space; and a discriminator $D$ that acts as a classifier
to distinguish \emph{real} samples of a dataset from \emph{fake} samples
produced by the generator $G$. Both $G$ and $D$ are parameterized
via neural networks, thus this method can be categorized into the
family of deep generative models or generative neural models \cite{goodfellow_etal_mit16_deeplearning}.

The optimization of GAN formulates a minimax problem, wherein given
an optimal $D$, the learning objective turns into finding $G$ that
minimizes the Jensen-Shannon divergence (JSD): $D_{\textrm{JS}}\left(p_{\textrm{data}}\Vert p_{\textrm{model}}\right)$.
The behavior of JSD minimization has been empirically proven to be
more similar to reverse KL than to KL divergence \cite{theis2015note,huszar2015not}.
This, however, leads to the aforementioned issue of mode collapse,
which is indeed a notorious failure of GAN \cite{goodfellow_nips17_gan_tutorial}
where the generator only produces similarly looking images, yielding
a low entropy distribution with poor variety of samples.

Recent attempts have been made to solve the mode collapsing problem
by improving the training of GAN. One idea is to use the minibatch
discrimination trick \cite{salimans_etal_nips16_improved} to allow
the discriminator to detect samples that are unusually similar to
other generated samples. Although this heuristics helps to generate
visually appealing samples very quickly, it is computationally expensive,
thus normally used in the last hidden layer of discriminator. Another
approach is to unroll the optimization of discriminator by several
steps to create a surrogate objective for the update of generator
during training \cite{metz2016unrolled}. The third approach is to train many
generators that discover different modes of the data \cite{hoang_etal_arxiv17_mggan}. Alternatively, around the same time, there
are various attempts to employ autoencoders as regularizers or auxiliary
losses to penalize missing modes \cite{che2016mode,warde2017improving,berthelot2017began,wang2017magan}.
These models can avoid the mode collapsing problem to a certain extent,
but at the cost of computational complexity with the exception of
DFM in \cite{warde2017improving}, rendering them \emph{unscalable}
up to ImageNet, a large-scale and challenging visual dataset.

Addressing these challenges, we propose a novel approach to both effectively
avoid mode collapse and efficiently scale up to very large datasets
(e.g., ImageNet). Our approach combines the KL and reverse KL divergences
into a unified objective function, thus it exploits the complementary
statistical properties from these divergences to effectively diversify
the estimated density in capturing multi-modes. We materialize our
idea using GAN's framework, resulting in a novel generative adversarial
architecture containing three players: a discriminator $D_{1}$ that
rewards high scores for data sampled from $p_{\textrm{data}}$ rather
than generated from the generator distribution $p_{G}$ whilst another
discriminator $D_{2}$, conversely, favoring data from $p_{G}$ rather
$p_{\textrm{data}}$, and a generator $G$ that generates data to
fool both two discriminators. We term our proposed model \emph{dual
discriminator generative adversarial network} ($\model$).

It turns out that training $\model$ shares the same minimax problem
as in GAN, which can be solved by alternatively updating the generator
and discriminators. We provide theoretical analysis showing that,
given $G$, $D_{1}$ and $D_{2}$ with enough capacity, i.e., in the
nonparametric limit, at the optimal points, the training criterion
indeed results in the minimal distance between data and model distribution
with respect to both their KL and reverse KL divergences. This helps
the model place fair distribution of probability mass across the modes
of the data generating distribution, thus allowing one to recover
the data distribution and generate diverse samples using the generator
in a single shot. In addition, we further introduce hyperparameters
to stabilize the learning and control the effect of each divergence.

We conduct extensive experiments on one synthetic dataset and four
real-world large-scale datasets (MNIST, CIFAR10, STL-10, ImageNet)
of very different nature. Since evaluating generative models is notoriously
hard \cite{theis2015note}, we have made our best effort to adopt
a number of evaluation metrics from literature to quantitatively compare
our proposed model with the latest state-of-the-art baselines whenever
possible. The experimental results reveal that our method is capable
of improving the diversity while keeping good quality of generated
samples. More importantly, our proposed model can be scaled up to
train on the large-scale ImageNet database, obtain a competitive variety
score and generate reasonably good quality images.

In short, our main contributions are: (i) a novel generative adversarial
model that encourages the diversity of samples produced by the generator;
(ii) a theoretical analysis to prove that our objective is optimized
towards minimizing both KL and reverse KL divergence and has a global
optimum where $p_{G}=p_{\textrm{data}}$; and (iii) a comprehensive
evaluation on the effectiveness of our proposed method using a wide
range of quantitative criteria on large-scale datasets.

\vspace{-2mm}

\section{Generative Adversarial Nets}

\vspace{-2mm}
We first review the generative adversarial network (GAN) that was
introduced in \cite{goodfellow_etal_nips14_gan} to formulate a game
of two players: a discriminator $D$ and a generator $G$. The discriminator,
$D\left(\bx\right)$, takes a point $\bx$ in data space and computes
the probability that $\bx$ is sampled from data distribution $p_{\textrm{data}}$,
rather than generated by the generator $G$. At the same time, the
generator first maps a noise vector $\bz$ drawn from a prior $p\left(\bz\right)$
to the data space, obtaining a sample $G\left(\bz\right)$ that resembles
the training data, and then uses this sample to challenge the discriminator.
The mapping $G\left(\bz\right)$ induces a generator distribution
$P_{G}$ in data domain with probability density function $p_{G}\left(\bx\right)$.
Both $G$ and $D$ are parameterized by neural networks (see Fig.~\ref{fig:GAN}
for an illustration) and learned by solving the following minimax
optimization:
\begin{align*}
\min_{G}\max_{D}\mathcal{J}\left(G,D\right) & =\mathbb{E}_{\vectorize x\sim P_{data}\left(\vectorize x\right)}\left[\log\left(D\left(\vectorize x\right)\right)\right]+\mathbb{E}_{\vectorize z\sim P_{\vectorize z}}\left[\log\left(1-D\left(G\left(\vectorize z\right)\right)\right)\right]
\end{align*}

The learning follows an iterative procedure wherein the discriminator
and generator are alternatively updated. Given a fixed $G$, the maximization
subject to $D$ results in the optimal discriminator $D^{\star}\left(\bx\right)=\frac{p_{\text{data}}\left(\vectorize x\right)}{p_{\text{data}}\left(\vectorize x\right)+p_{G}\left(\vectorize x\right)}$
, whilst given this optimal $D^{\star}$, the minimization of $G$
turns into minimizing the Jensen-Shannon (JS) divergence between the
data and model distributions: $D_{\textrm{JS}}\left(P_{\textrm{data}}\Vert P_{G}\right)$
\cite{goodfellow_etal_nips14_gan}. At the Nash equilibrium of a game,
the model distribution recovers the data distribution exactly: $P_{G}=P_{\textrm{data}}$,
thus the discriminator $D$ now fails to differentiate real or fake
data as $D\left(\bx\right)=0.5,\forall\bx$.

\begin{figure}[h]
\begin{centering}
\vspace{-7mm}
\subfloat[GAN.\label{fig:GAN}]{\noindent \centering{}\includegraphics[width=0.25\textwidth]{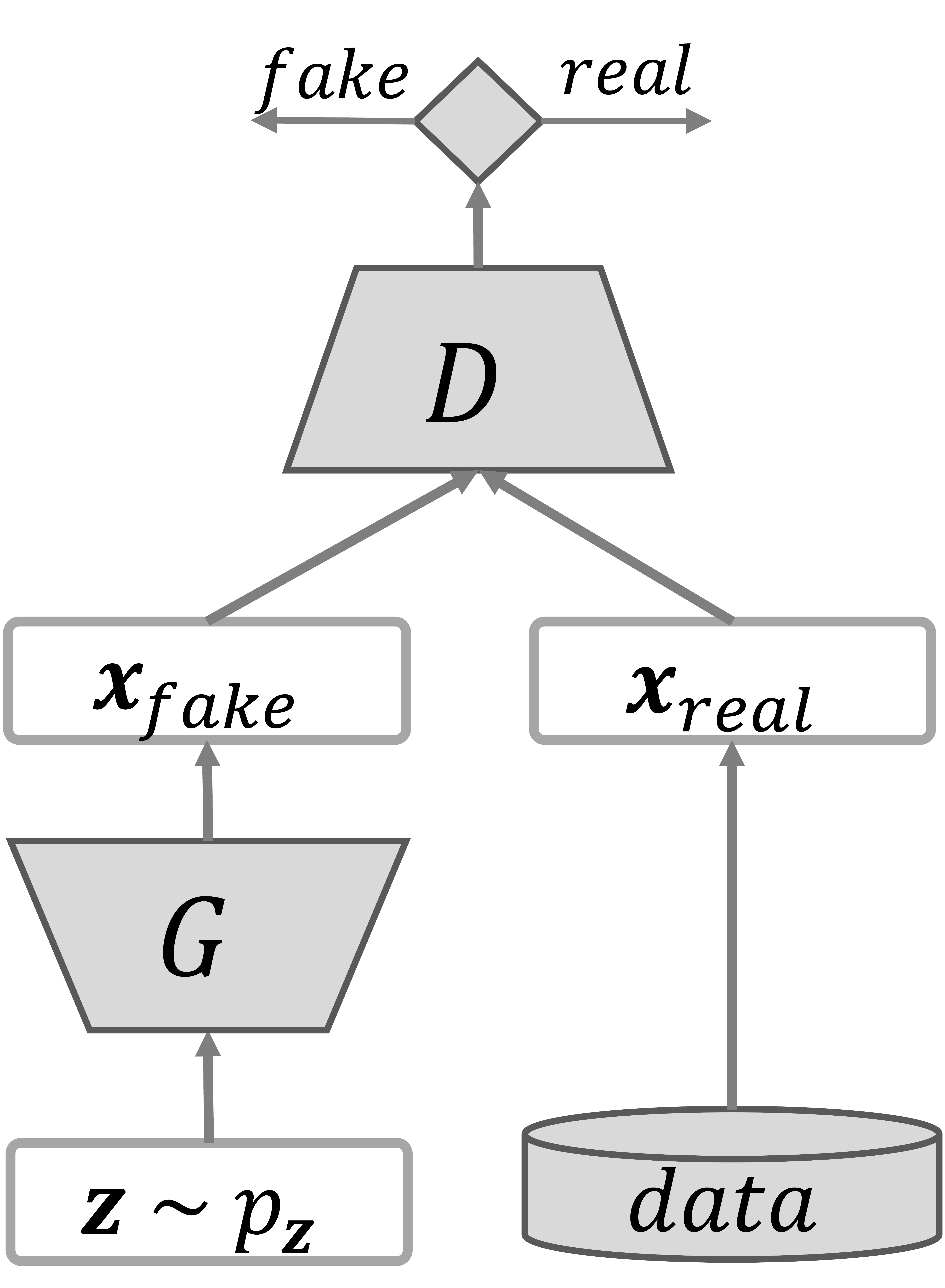}}\hspace{20mm}\subfloat[$\protect\model$.\label{fig:D2GAN}]{\noindent \centering{}\includegraphics[width=0.3\textwidth]{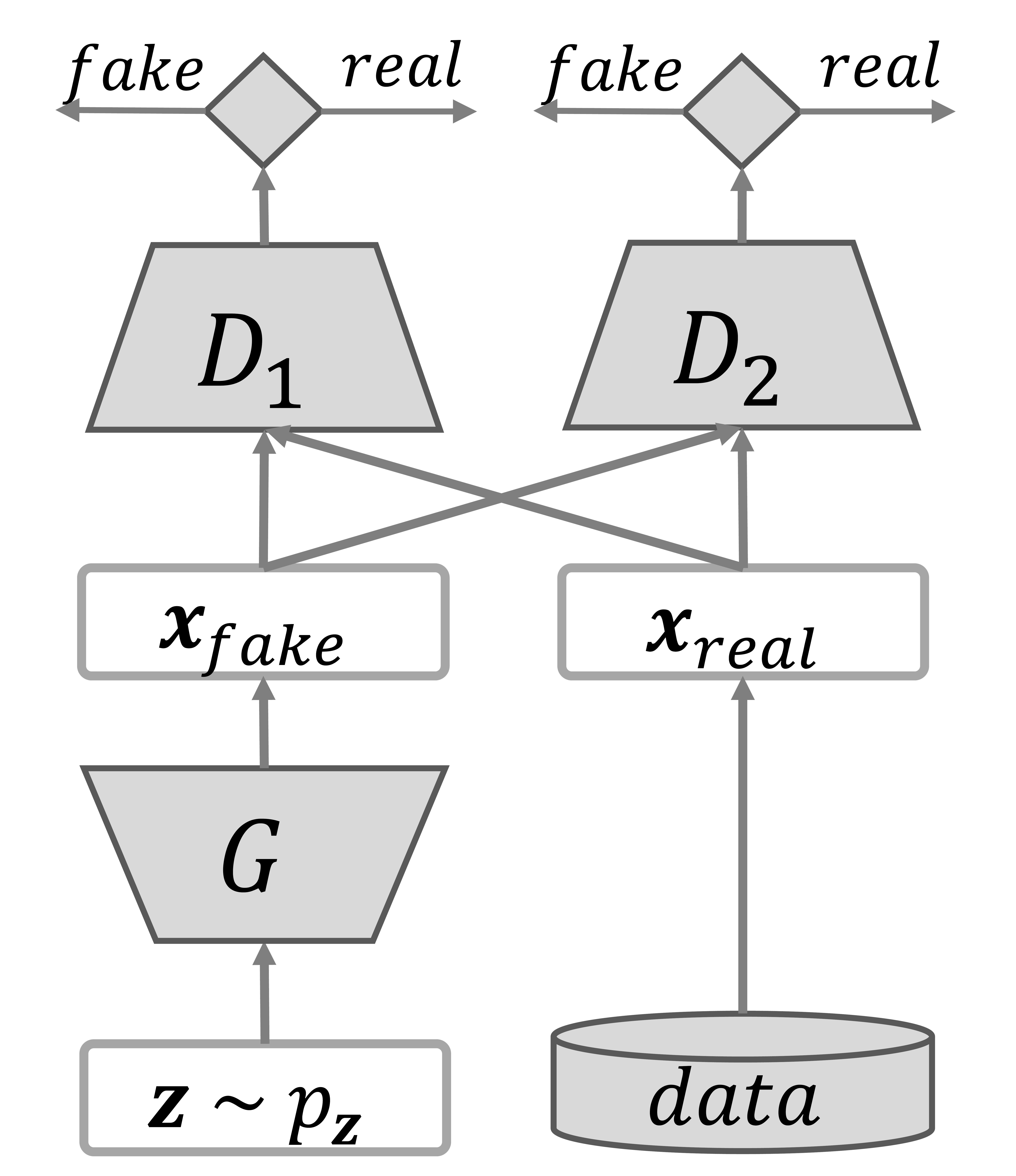}}\vspace{-1mm}

\par\end{centering}

\caption{An illustration of the standard GAN and our proposed $\protect\model$.\label{fig:GAN_vs_D2GAN}}
\vspace{-2mm}
\end{figure}

Since the JS divergence has been empirically proven to have the same
nature as that of the reverse KL divergence \cite{theis2015note,huszar2015not,goodfellow_nips17_gan_tutorial},
GAN suffers from the model collapsing problem, and thus its generated
data samples have low level of diversity \cite{metz2016unrolled,che2016mode}.

\section{Dual Discriminator Generative Adversarial Nets}

\vspace{-2mm}

To tackle GAN's problem of mode collapse, in what follows we present
our main contribution of a framework that seeks an approximated distribution
to effectively cover many modes of the multimodal data. Our intuition
is based on GAN, but we formulate a three-player game that consists
of two different discriminators $D_{1}$ and $D_{2}$, and one generator
$G$. Given a sample $\vectorize x$ in data space, $D_{1}\left(\vectorize x\right)$
rewards a high score if $\vectorize x$ is drawn from the data distribution
$P_{\text{data}}$, and gives a low score if generated from the model
distribution $P_{G}$. In contrast, $D_{2}\left(\vectorize x\right)$
returns a high score for $\bx$ generated from $P_{G}$ whilst giving
a low score for a sample drawn from $P_{\textrm{data}}$. Unlike GAN,
the scores returned by our discriminators are values in $\realset^{+}$
rather than probabilities in $\left[0,1\right]$. Our generator $G$
performs a similar role to that of GAN, i.e., producing data mapped
from a noise space to synthesize the real data and then fool both
two discriminators $D_{1}$ and $D_{2}$. All three players are parameterized
by neural networks wherein $D_{1}$ and $D_{2}$ do not share their
parameters. We term our proposed model\emph{ dual discriminator generative
adversarial network} ($\model$). Fig.~\ref{fig:D2GAN} shows an
illustration of $\model$.

More formally, $D_{1}$, $D_{2}$ and $G$ now play the following
three-player minimax optimization game:
\begin{align}
\min_{G}\max_{D_{1},D_{2}}\mathcal{J}\left(G,D_{1},D_{2}\right) & =\alpha\times\mathbb{E}_{\bx\sim P_{\text{data}}}\left[\log D_{1}\left(\vectorize x\right)\right]+\mathbb{E}_{\bz\sim P_{\bz}}\left[-D_{1}\left(G\left(\vectorize z\right)\right)\right]\nonumber \\
 & \,\,\,\,\,\,+\mathbb{E}_{\bx\sim P_{\text{data}}}\left[-D_{2}\left(\vectorize x\right)\right]+\beta\times\mathbb{E}_{\bz\sim P_{\bz}}\left[\log D_{2}\left(G\left(\vectorize z\right)\right)\right]\label{eq:D2GAN_loss}
\end{align}
wherein we have introduced hyperparameters $0<\alpha,\beta\leq1$
to serve two purposes. The first is to stabilize the learning of our
model. As the output values of two discriminators are positive and
unbounded, $D_{1}\left(G\left(\bz\right)\right)$ and $D_{2}\left(\vectorize x\right)$
in Eq.~(\ref{eq:D2GAN_loss}) can become very large and have exponentially
stronger impact on the optimization than $\log D_{1}\left(\vectorize x\right)$
and $\log D_{2}\left(G\left(\vectorize z\right)\right)$ do, rendering
the learning unstable. To overcome this issue, we can decrease $\alpha$
and $\beta$, in effect making the optimization penalize $D_{1}\left(G\left(\bz\right)\right)$
and $D_{2}\left(\vectorize x\right)$, thus helping to stabilize the
learning. The second purpose of introducing $\alpha$ and $\beta$
is to control the effect of KL and reverse KL divergences on the optimization
problem. This will be discussed in the following part once we have
the derivation of our optimal solution.

Similar to GAN \cite{goodfellow_etal_nips14_gan}, our proposed network
can be trained by alternatively updating $D_{1},D_{2}$ and $G$.
We refer to the supplementary material for the pseudo-code of learning
parameters for $\model$.\vspace{-1mm}

\subsection{Theoretical analysis}

\vspace{-2mm}
We now provide formal theoretical analysis of our proposed model,
that essentially shows that, given $G$, $D_{1}$ and $D_{2}$ are
of enough capacity, i.e., in the nonparametric limit, at the optimal
points, $G$ can recover the data distributions by minimizing both
KL and reverse KL divergences between model and data distributions.
We first consider the optimization problem with respect to (w.r.t)
discriminators given a fixed generator.
\begin{prop}
\label{thm:D2GAN_opt_d}Given a fixed $G$, maximizing $\mathcal{J}\left(G,D_{1},D_{2}\right)$
yields to the following closed-form optimal discriminators $D_{1}^{\star},D_{2}^{\star}$\emph{:
\begin{eqnarray*}
D_{1}^{\star}\left(\vectorize x\right)=\frac{\alpha p_{\text{data}}\left(\vectorize x\right)}{p_{G}\left(\vectorize x\right)} & \textrm{and} & D_{2}^{\star}\left(\vectorize x\right)=\frac{\beta p_{G}\left(\vectorize x\right)}{p_{\text{data}}\left(\vectorize x\right)}
\end{eqnarray*}
}\end{prop}
\begin{proof}
According to the induced measure theorem \cite{gupta_shao_jstor00_mathematical},
two expectations are equal: $\mathbb{E}_{\bz\sim P_{\bz}}\left[f\left(G\left(\vectorize z\right)\right)\right]=\expect_{\bx\sim P_{G}}\left[f\left(\bx\right)\right]$
where $f\left(\bx\right)=-D_{1}\left(\bx\right)$ or $f\left(\bx\right)=\log D_{2}\left(\bx\right)$.
The objective function can be rewritten as below:
\begin{align*}
\mathcal{J}\left(G,D_{1},D_{2}\right) & =\alpha\times\mathbb{E}_{\bx\sim P_{\text{data}}}\left[\log D_{1}\left(\vectorize x\right)\right]+\mathbb{E}_{\bx\sim P_{G}}\left[-D_{1}\left(\bx\right)\right]\\
 & \,\,\,\,\,\,+\mathbb{E}_{\bx\sim P_{\text{data}}}\left[-D_{2}\left(\vectorize x\right)\right]+\beta\times\mathbb{E}_{\bx\sim P_{G}}\left[\log D_{2}\left(\bx\right)\right]\\
 & =\int_{\bx}\left[\alpha p_{\text{data}}\left(\vectorize x\right)\log D_{1}\left(\vectorize x\right)-p_{G}D_{1}\left(\bx\right)-p_{\textrm{data}}\left(\bx\right)D_{2}\left(\bx\right)+\beta p_{G}\log D_{2}\left(\bx\right)\right]\d\vectorize x
\end{align*}
Considering the function inside the integral, given $\vectorize x$,
we maximize this function w.r.t two variables $D_{1},D_{2}$ to find
$D_{1}^{\star}\left(\vectorize x\right)$ and $D_{2}^{\star}\left(\vectorize x\right)$.
Setting the derivatives w.r.t $D_{1}$ and $D_{2}$ to $0$, we gain:
\begin{eqnarray}
\frac{\alpha p_{\text{data}}\left(\vectorize x\right)}{D_{1}}-p_{G}\left(\vectorize x\right)=0 & \textrm{and} & \frac{\beta p_{G}\left(\vectorize x\right)}{D_{2}}-p_{\text{data}}\left(\vectorize x\right)=0\label{eq:D2GAN_opt_d}
\end{eqnarray}
The second derivatives: $\nicefrac{-\alpha p_{\textrm{data}}\left(\bx\right)}{D_{1}^{2}}$
and $\nicefrac{-\beta p_{G}\left(\bx\right)}{D_{2}^{2}}$ are non-positive,
thus verifying that we have obtained the maximum solution and concluding
the proof.
\end{proof}
\vspace{-1mm}
Next, we fix $D_{1}=D_{1}^{\star},D_{2}=D_{2}^{\star}$ and find the
optimal solution $G^{\star}$ for the generator $G$.

\begin{thm}
\label{thm:D2GAN_opt_g}Given $D_{1}^{\star},D_{2}^{\star}$, at the
Nash equilibrium point $\left(G^{\star},D_{1}^{\star},D_{2}^{\star}\right)$
for minimax optimization problem of $\model$, we have the following
form for each component:
\begin{align*}
\mathcal{J}\left(G^{\star},D_{1}^{\star},D_{2}^{\star}\right) & =\alpha\left(\log\alpha-1\right)+\beta\left(\log\beta-1\right)\\
D_{1}^{\star}\left(\vectorize x\right) & =\alpha\,\,\textrm{and}\,\,D_{2}^{\star}\left(\vectorize x\right)=\beta,\forall\vectorize x\,\,\textrm{at}\,\,p_{G^{\star}}=p_{\textrm{data}}
\end{align*}
\end{thm}
\begin{proof}
Substituting $D_{1}^{\star},D_{2}^{\star}$ from Eq.~(\ref{eq:D2GAN_opt_d})
into the objective function in Eq.~(\ref{eq:D2GAN_loss}) of the
minimax problem, we gain:
\begin{align}
\mathcal{J}\left(G,D_{1}^{\star},D_{2}^{\star}\right) & =\alpha\times\mathbb{E}_{\bx\sim P_{\text{data}}}\left[\log\alpha+\log\frac{p_{\text{data}}\left(\vectorize x\right)}{p_{G}\left(\vectorize x\right)}\right]-\alpha\int_{\bx}p_{G}\left(\vectorize x\right)\frac{p_{\textrm{data}}\left(\vectorize x\right)}{p_{G}\left(\vectorize x\right)}\d\vectorize x\nonumber \\
 & \,\,\,\,\,\,-\beta\int_{\bx}p_{\textrm{data}}\frac{p_{G}\left(\bx\right)}{p_{\textrm{data}}\left(\bx\right)}\d\bx+\beta\times\mathbb{E}_{\bx\sim P_{G}}\left[\log\beta+\log\frac{p_{G}\left(\vectorize x\right)}{p_{\text{data}}\left(\vectorize x\right)}\right]\nonumber \\
 & =\alpha\left(\log\alpha-1\right)+\beta\left(\log\beta-1\right)+\alpha D_{\textrm{KL}}\left(P_{\text{data}}\Vert P_{G}\right)+\beta D_{\textrm{KL}}\left(P_{G}\Vert P_{\text{data}}\right)\label{eq:D2GAN_G_loss}
\end{align}
where $D_{\textrm{KL}}\left(P_{\text{data}}\Vert P_{G}\right)$ and
$D_{\textrm{KL}}\left(P_{G}\Vert P_{\text{data}}\right)$ is the KL
and reverse KL divergences between data and model (generator) distributions,
respectively. These divergences are always nonnegative and only zero
when two distributions are equal: $p_{G^{\star}}=p_{\textrm{data}}$.
In other words, the generator induces a distribution $p_{G^{\star}}$
that is identical to the data distribution $p_{\textrm{data}}$, and
two discriminators now fail to recognize the real or fake samples
since they return the same score of 1 for both samples. This concludes
the proof.
\end{proof}

\vspace{-2mm}
The loss of generator in Eq.~(\ref{eq:D2GAN_G_loss}) shows that
increasing $\alpha$ promotes the optimization towards minimizing
the KL divergence $D_{\textrm{KL}}\left(P_{\text{data}}\Vert P_{G}\right)$,
thus helping the generative distribution cover multiple modes, but
may include potentially undesirable samples; whereas increasing $\beta$
encourages the minimization of the reverse KL divergence $D_{\textrm{KL}}\left(P_{G}\Vert P_{\text{data}}\right)$,
hence enabling the generator capture a single mode better, but may
miss many modes. By empirically adjusting these two hyperparameters,
we can balance the effect of two divergences, and hence effectively
avoid the mode collapsing issue.

\vspace{-2mm}

\section{Experiments}

\vspace{-3mm}
In this section, we conduct comprehensive experiments to demonstrate
the capability of improving mode coverage and the scalability of our
proposed model on large-scale datasets. We use a synthetic 2D dataset
for both visual and numerical verification, and four datasets of increasing
diversity and size for numerical verification. We have made our best
effort to compare the results of our method with those of the latest
state-of-the-art GAN's variants by replicating experimental settings
in the original work whenever possible.

For each experiment, we refer to the supplementary material for model
architectures and additional results. Common points are: i) discriminators'
outputs with \emph{softplus} activations :$f\left(x\right)=\ln\left(1+e^{x}\right)$,
i.e., positive version of ReLU; (ii) Adam optimizer \cite{kingma2014adam}
with learning rate 0.0002 and the first-order momentum 0.5; (iii)
minibatch size of 64 samples for training both generator and discriminators;
(iv) Leaky ReLU with the slope of 0.2; and (v) weights initialized
from an isotropic Gaussian: $\mathcal{N}\left(0,0.01\right)$ and
zero biases. Our implementation is in TensorFlow \cite{tensorflow2015-whitepaper}
and will be released once published. We now present our experiments
on synthetic data followed by those on large-scale real-world datasets.\vspace{-1mm}

\subsection{Synthetic data}

\vspace{-2mm}
In the first experiment, we reuse the experimental design proposed
in \cite{metz2016unrolled} to investigate how well our $\model$
can deal with multiple modes in the data. More specifically, we sample
training data from a 2D mixture of 8 Gaussian distributions with a
covariance matrix 0.02$\bI$ and means arranged in a circle of zero
centroid and radius $2.0$. Data in these low variance mixture components
are separated by an area of very low density. The aim is to examine
properties such as low probability regions and low separation of modes.

We use a simple architecture of a generator with two fully connected
hidden layers and discriminators with one hidden layer of ReLU activations.
This setting is identical, thus ensures a fair comparison with UnrolledGAN\footnote{We obtain the code of UnrolledGAN for 2D data from the link authors
provided in \cite{metz2016unrolled}.} \cite{metz2016unrolled}. Fig.~\ref{fig:2D_distributions} shows
the evolution of 512 samples generated by our models and baselines
through time. It can be seen that the regular GAN generates data collapsing
into a \emph{single} mode hovering around the valid modes of data
distribution, thus reflecting the mode collapse in GAN. At the same
time, UnrolledGAN and $\model$ distribute data around \emph{all}
8 mixture components, and hence demonstrating the abilities to successfully
learn multimodal data in this case. At the last steps, our $\model$
captures data modes more precisely than UnrolledGAN as, in each mode,
the UnrolledGAN generates data that concentrate only on \emph{several}
points around the mode's centroid, thus seems to produce fewer samples
than $\model$ whose samples fairly spread out the \emph{entire} mode.

Next we further quantitatively compare the quality of generated data.
Since we know the true distribution $p_{\textrm{data}}$ in this case,
we employ two measures, namely symmetric KL divergence and Wasserstein
distance. These measures compute the distance between the normalized
histograms of 10,000 points generated from our $\model$, UnrolledGAN
and GAN to true $p_{\textrm{data}}$. Figs.~\ref{fig:2D_symmetric_KL}
and \ref{fig:2D_Wasserstein} again clearly demonstrate the superiority
of our approach over GAN and UnrolledGAN w.r.t both distances (lower
is better); notably with Wasserstein metric, the distance from ours
to the true distribution almost reduces to zero. These figures also
demonstrate the stability of our $\model$ (red curves) during training
as it is much less fluctuating compared with GAN (green curves) and
UnrolledGAN (blue curves).

\begin{figure}
\noindent \begin{centering}
\begin{minipage}[t]{0.3\textwidth}%
\noindent \begin{center}
\subfloat[Symmetric KL divergence.\label{fig:2D_symmetric_KL}]{\noindent \begin{centering}
\includegraphics[width=1\textwidth]{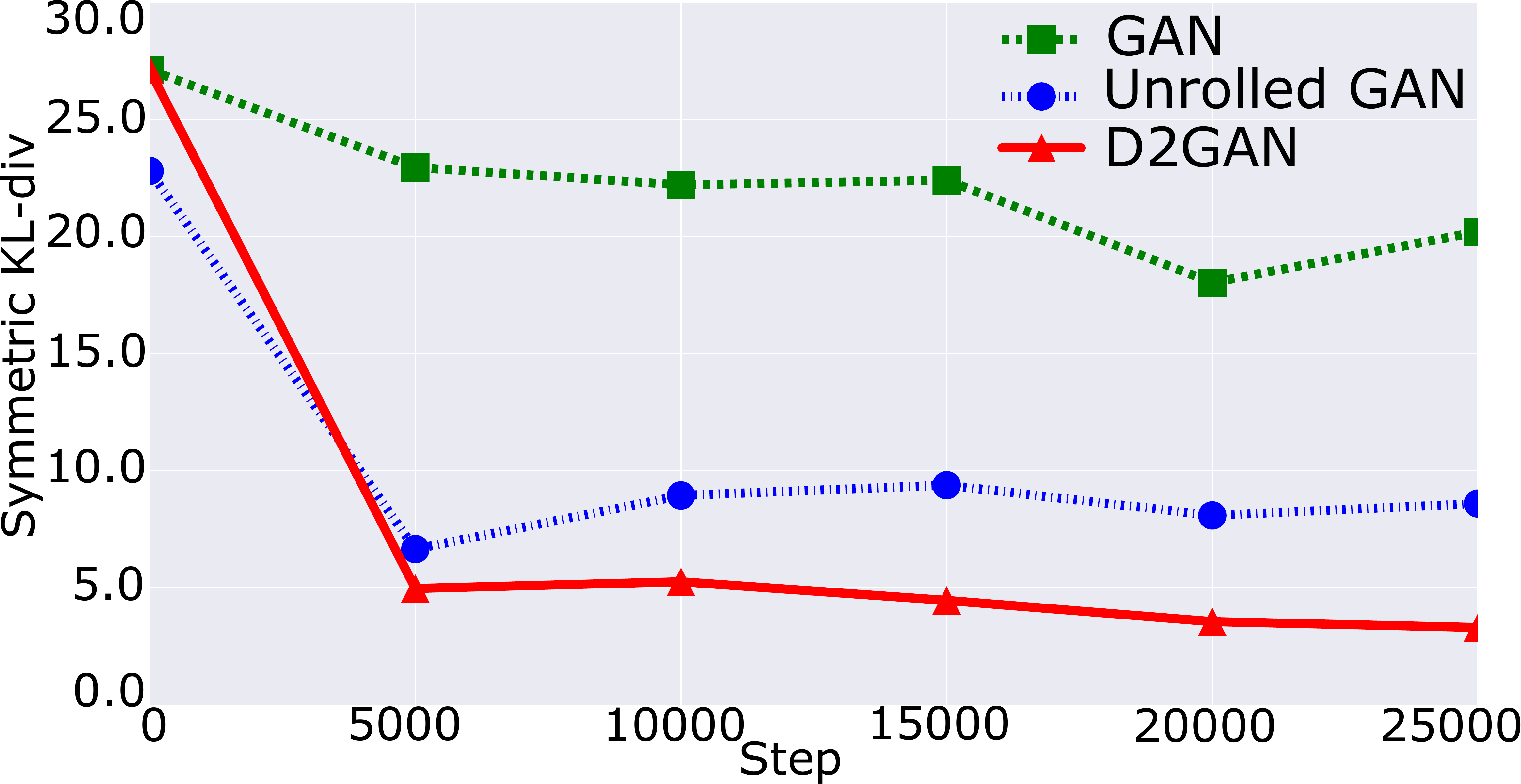}
\par\end{centering}

\noindent \centering{}}
\par\end{center}

\noindent \begin{center}
\subfloat[Wasserstein distance.\label{fig:2D_Wasserstein}]{\noindent \begin{centering}
\includegraphics[width=1\textwidth]{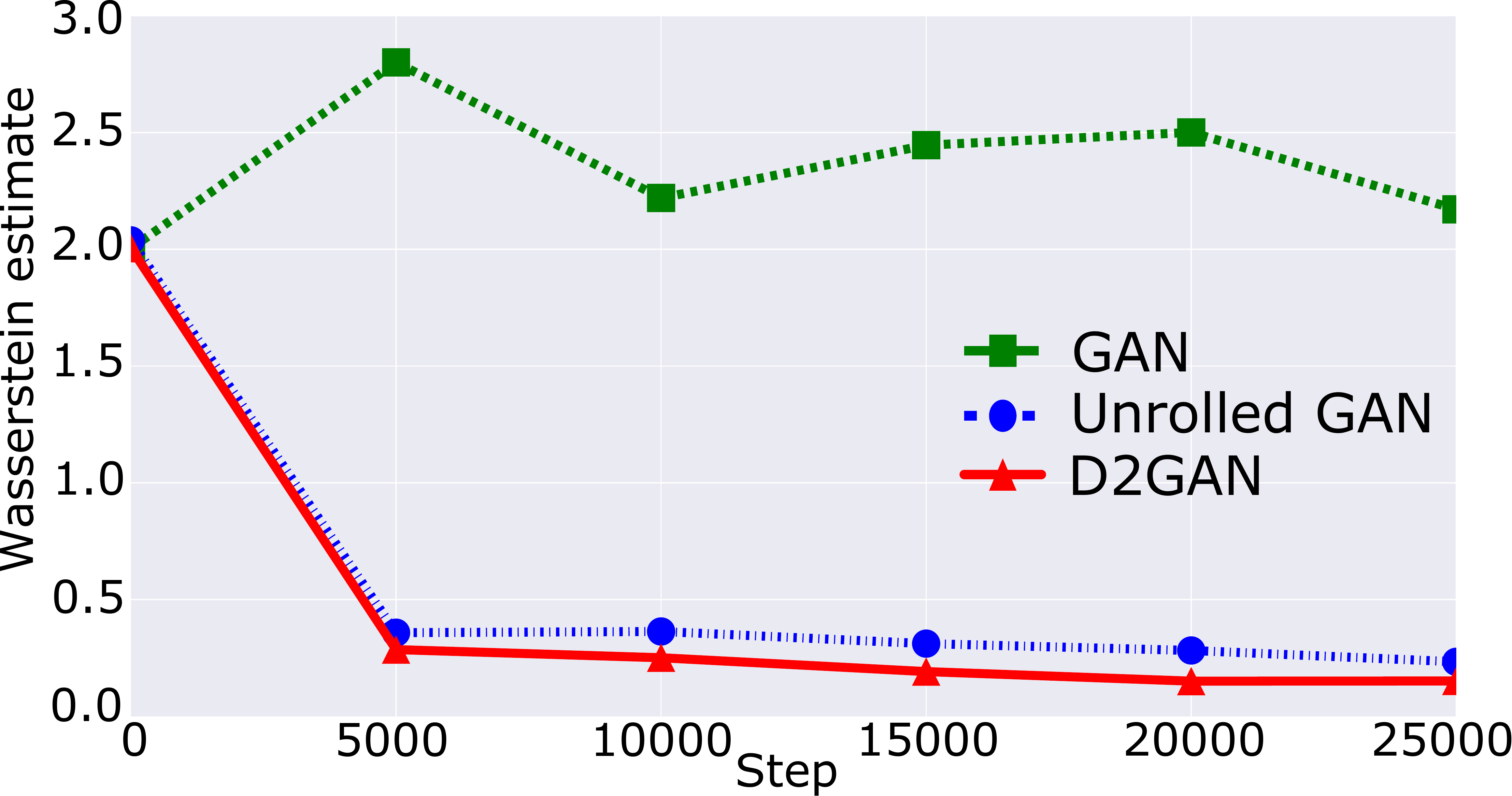}
\par\end{centering}

\noindent \centering{}}
\par\end{center}%
\end{minipage}\hfill{}%
\begin{minipage}[t]{0.65\columnwidth}%
\noindent \begin{center}
\subfloat[Evolution of data (in blue) generated from GAN (top row), UnrolledGAN
(middle row) and our $\protect\model$ (bottom row) on 2D data of
$8$ Gaussians. Data sampled from the true mixture are red.\label{fig:2D_distributions}]{\noindent \centering{}\includegraphics[width=1\textwidth]{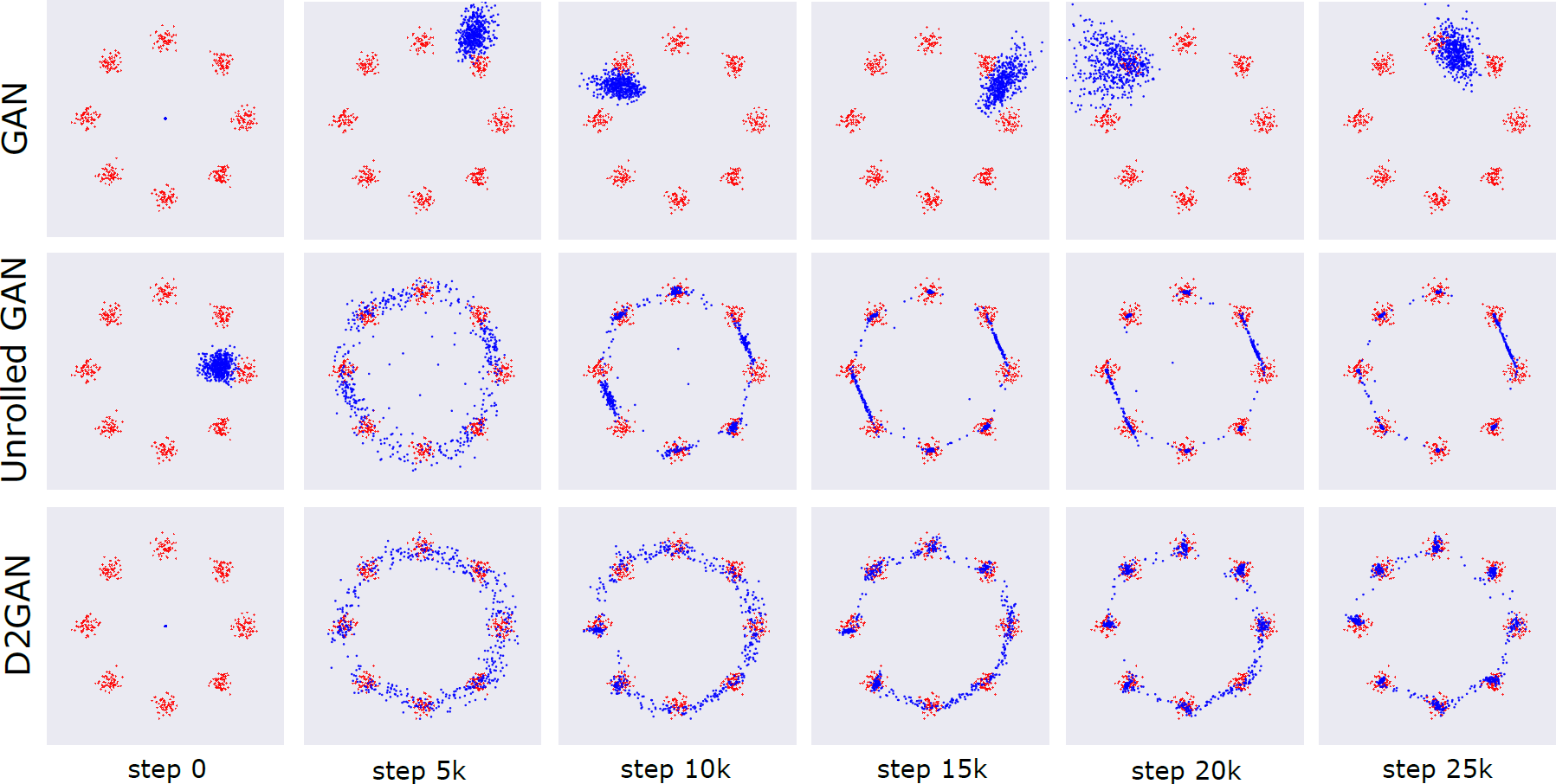}}
\par\end{center}%
\end{minipage}
\par\end{centering}

\noindent \centering{}\vspace{-2mm}
\caption{The comparison of standard GAN, UnrolledGAN and our $\protect\model$
on 2D synthetic dataset.}
\vspace{-5mm}
\end{figure}
\vspace{-2mm}

\subsection{Real-world datasets}

\vspace{-2mm}
We now examine the performance of our proposed method on real-world
datasets with increasing diversities and sizes. For networks containing
convolutional layers, we closely follow the DCGAN's design \cite{radford_etal_iclr15_unsupervised}.
We use strided convolutions for discriminators and fractional-strided
convolutions for generator instead of pooling layers. Batch normalization
is applied for each layer, except the generator output layer and the
discriminator input layers. We also use Leaky ReLU activations for
discriminators, and use ReLU for generator, except its output is \emph{tanh}
since we rescale the pixel intensities into the range of {[}-1, 1{]}
before feeding images to our model. Only one difference is that, for
our model, initializing the weights from $\mathcal{N}\left(0,0.01\right)$
yields slightly better results than from $\mathcal{N}\left(0,0.02\right)$.
We again refer to the supplementary material for detailed architectures.\vspace{-1mm}

\subsubsection{Evaluation protocol}

\vspace{-2mm}
Evaluating the quality of image produced by generative models is a
notoriously challenging due to the variety of probability criteria
and the lack of a perceptually meaningful image similarity metric
\cite{theis2015note}. Even when a model can generate plausible images,
it is not useful if those images are visually similar. Therefore,
in order to quantify the performance of covering data modes as well
as producing high quality samples, we use several different ad-hoc
metrics for different experiments to compare with other baselines.

First we adopt the \emph{Inception score} proposed in \cite{salimans_etal_nips16_improved},
which are computed by: $\exp\left(\expect_{\bx}\left[D_{\textrm{KL}}\left(p\left(\y\gv\bx\right)\parallel p\left(\y\right)\right)\right]\right)$,
where $p\left(\y\gv\bx\right)$ is the conditional label distribution
for image $\bx$ estimated using a pretrained Inception model \cite{szegedy_etal_cvpr16_rethinking},
and $p\left(\y\right)$ is the marginal distribution: $p\left(\y\right)\approx\nicefrac{1}{\N}\sum_{n=1}^{\N}p\left(\y\gv\bx_{n}=G\left(\bz_{n}\right)\right)$.
This metric rewards good and varied samples, but sometimes is easily
fooled by a model that collapses and generates to a very low quality
image, thus fails to measure whether a model has been trapped into
one bad mode. To address this problem, for labeled datasets, we further
recruit the so-called MODE score introduced in \cite{che2016mode}:
\[
\exp\left(\expect_{\bx}\left[D_{\textrm{KL}}\left(p\left(\y\gv\bx\right)\parallel\tilde{p}\left(\y\right)\right)\right]-D_{\textrm{KL}}\left(p\left(\y\right)\parallel\tilde{p}\left(\y\right)\right)\right)
\]
 where $\tilde{p}\left(\y\right)$ is the empirical distribution of
labels estimated from training data. The score can adequately reflect
the variety and visual quality of images, which is discussed in \cite{che2016mode}.\vspace{-1mm}

\subsubsection{Handwritten digit images}

\vspace{-2mm}
We start with the handwritten digit images -- MNIST \cite{lecun_etal_mnist}
that consists of 60,000 training and 10,000 testing 28$\times$28
grayscale images of digits from 0 to 9. Following the setting in \cite{che2016mode},
we first assume that the MNIST has 10 modes, representing connected
component in the data manifold, associated with 10 digit classes.
We then also perform an extensive grid search of different hyperparameter
configurations, wherein our two regularized constants $\alpha,\beta$
in Eq.~(\ref{eq:D2GAN_loss}) are varied in \{0.01, 0.05, 0.1, 0.2\}.
For a fair comparison, we use the same parameter ranges and fully
connected layers for our network (c.f. the supplementary material
for more details), and adopt results of GAN and mode regularized GAN
(Reg-GAN) from \cite{che2016mode}.

For evaluation, we first train a simple, yet effective 3-layer convolutional
nets\footnote{Network architecture is similar to \href{https://github.com/fchollet/keras/blob/master/examples/mnist_cnn.py}{https://github.com/fchollet/keras/blob/master/examples/mnist\_{}cnn.py}.}
that can obtain 0.65\% error on MNIST testing set, and then employ
it to predict the label probabilities and compute MODE scores for
generated samples. Fig.~\ref{fig:exp_mnist_mode_score} (left) shows
the distributions of MODE scores obtained by three models. Clearly,
our proposed $\model$ significantly outperforms the standard GAN
and Reg-GAN by achieving scores mostly around the maximum {[}8.0-9.0{]}.
It is worthy to note that we did not observe substantial differences
in the average MODE scores obtained by varying the network size through
the parameter searching. We here report the result of the minimal
network with the smallest number of layers and hidden units.

To study the effect of $\alpha$ and $\beta$, we inspect the results
obtained by this minimal network with varied $\alpha,\beta$ in Fig.~\ref{fig:exp_mnist_mode_score}
(right). There is a pattern that, given a fixed $\alpha$, our $\model$
obtains better MODE score when increasing $\beta$ to a certain value,
after which the score could significantly decrease.

\vspace{-4mm}
\begin{figure}[h]
\noindent \centering{}\includegraphics[width=0.72\textwidth]{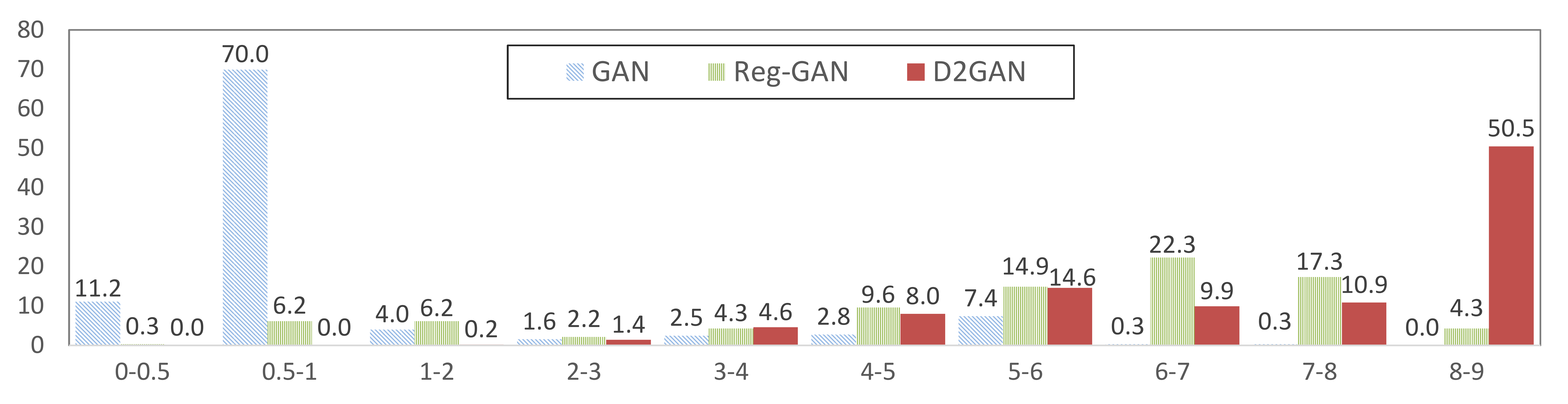}\includegraphics[width=0.28\textwidth]{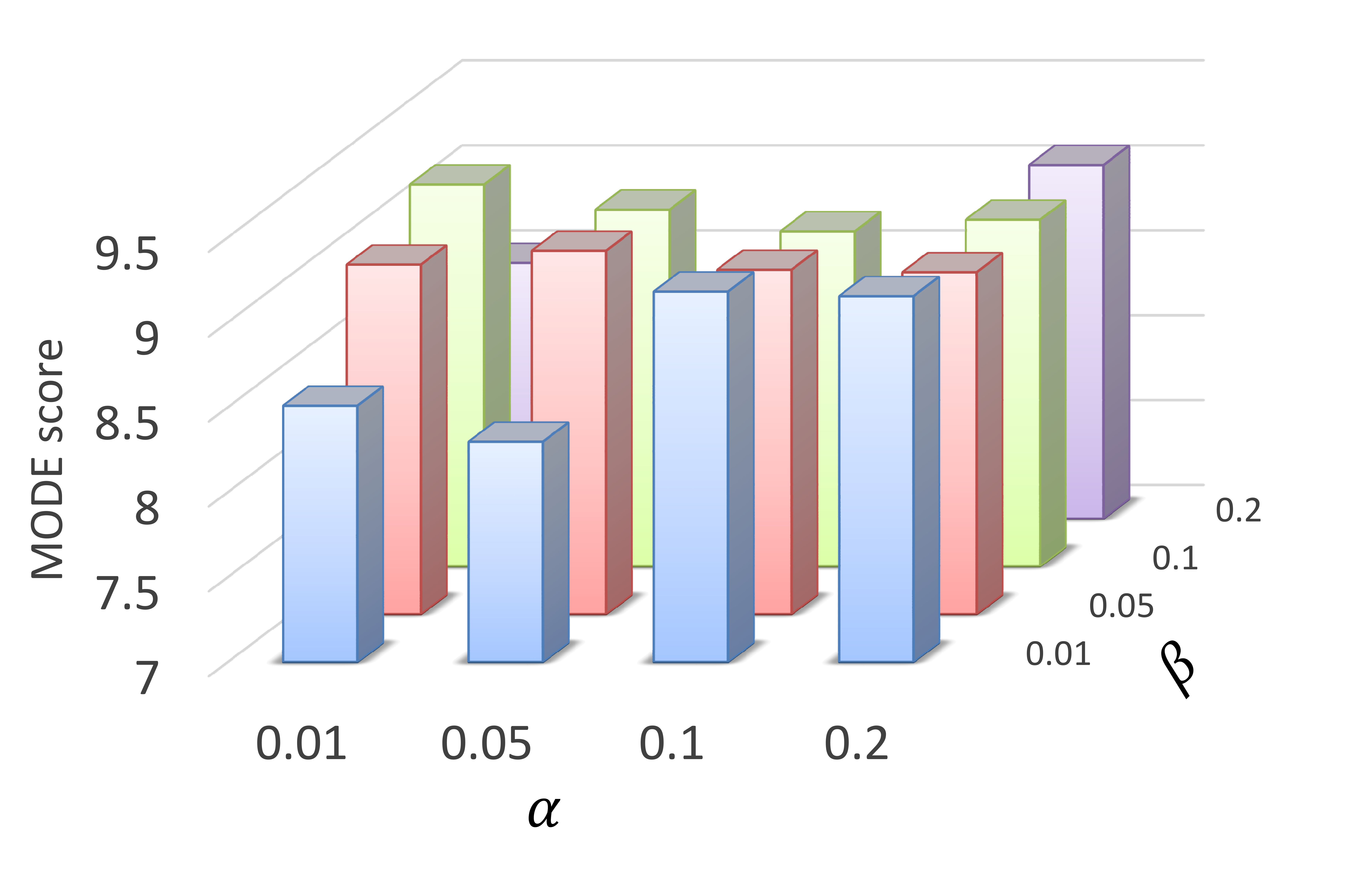}\vspace{-3mm}
\caption{Distributions of MODE scores (left) and average MODE scores (right)
with varied $\alpha$, $\beta$.\label{fig:exp_mnist_mode_score}}
\end{figure}
\vspace{-4mm}

\paragraph{MNIST-1K.}

The standard MNIST data with 10-mode assumption seems to be fairly
trivial. Hence, based on this data, we test our proposed model on
a more challenging one. We continue following the technique used in
\cite{che2016mode,metz2016unrolled} to construct a new 1000-class
MNIST dataset (MNIST-1K) by stacking three randomly selected digits
to form an RGB image with a different digit image in each channel.
The resulting data can be assumed to contain 1,000 distinct modes,
corresponding to the combinations of digits in 3 channels from 000
to 999.

In this experiment, we use a more powerful model with convolutional
layers for discriminators and transposed convolutions for the generator.
We measure the performance by the number of modes for which the model
generated at least one in total 25,600 samples, and the reverse KL
divergence between the model distribution (i.e., the label distribution
predicted by the pretrained MNIST classifier used in the previous
experiment) and the expected data distribution. Tab.~\ref{tab:exp_mnist1k}
reports the results of our $\model$ compared with those of GAN, UnrolledGAN
taken from \cite{metz2016unrolled}, DCGAN and Reg-GAN from \cite{che2016mode}.
Our proposed method again clearly demonstrates the superiority over
baselines by covering all modes and achieving the best distance that
is close to zero.

\vspace{-2mm}
\begin{table}[h]
\noindent \centering{}\caption{Numbers of modes covered and reverse KL divergence between model and
data distributions.\label{tab:exp_mnist1k}}
\resizebox{0.99\textwidth}{!}{
\begin{tabular}{crrrrr}
\hline 
Model & GAN \cite{metz2016unrolled} & UnrolledGAN \cite{metz2016unrolled} & DCGAN \cite{che2016mode} & Reg-GAN \cite{che2016mode} & $\model$\tabularnewline
\hline 
\# modes covered & 628.0$\pm$140.9 & 817.4$\pm$37.9 & 849.6$\pm$62.7 & 955.5$\pm$18.7 & \textbf{1000.0$\pm$0.00}\tabularnewline
$D_{\textrm{KL}}\left(\textrm{model\ensuremath{\parallel\textrm{data}}}\right)$ & 2.58$\pm$0.75 & 1.43$\pm$0.12 & 0.73$\pm$0.09 & 0.64$\pm$0.05 & \textbf{0.08$\pm$0.01}\tabularnewline
\hline 
\end{tabular}}
\end{table}
\vspace{-2mm}

\subsubsection{Natural scene images}

\vspace{-2mm}
We now extend our experiments to investigate the scalability of our
proposed method on much more challenging large-scale image databases
from natural scenes. We use three widely-adopted datasets: CIFAR-10
\cite{krizhevsky_thesis09_learning}, STL-10 \cite{coates_etal_aistats11_analysis}
and ImageNet \cite{ILSVRC15}. CIFAR-10 is a well-studied dataset
of 50,000 32$\times$32 training images of 10 classes: airplane, automobile,
bird, cat, deer, dog, frog, horse, ship, and truck. STL-10, a subset
of ImageNet, contains about 100,000 unlabeled 96$\times$96 images,
which is more diverse than CIFAR-10, but less so than the full ImageNet.
We rescale all images down 3 times and train our networks on 32$\times$32
resolution. ImageNet is a very large database of about 1.2 million
natural images from 1,000 classes, normally used as the most challenging
benchmark to validate the scalability of deep models. We follow the
preprocessing in \cite{krizhevsky_etal_nips12_imagenet}, except subsampling
to 32$\times$32 resolution. We use the code provided in \cite{salimans_etal_nips16_improved}
to compute the Inception score for 10 independent partitions of 50,000
generated samples.

\begin{figure}
\noindent \centering{}%
\begin{minipage}[t]{0.39\textwidth}%
\noindent \begin{center}
\captionof{table}{Inception scores on CIFAR-10.}\label{tab:exp_inception_scores_cifar10}%
\begin{tabular}{lr}
\hline 
Model & Score\tabularnewline
\hline 
Real data & 11.24$\pm$0.16\tabularnewline
\rowcolor{even_color}WGAN \cite{arjovsky_etal_arxiv17_wasserstein_gan} & 3.82$\pm$0.06\tabularnewline
MIX+WGAN \cite{arora2017generalization} & 4.04$\pm$0.07\tabularnewline
\rowcolor{even_color}Improved-GAN \cite{salimans_etal_nips16_improved} & 4.36$\pm$0.04\tabularnewline
ALI \cite{dumoulin2016adversarially} & 5.34\textbf{\emph{$\pm$}}0.05\tabularnewline
\rowcolor{even_color}BEGAN \cite{berthelot2017began} & 5.62~~~~~~~~~~~\tabularnewline
MAGAN \cite{wang2017magan} & 5.67~~~~~~~~~~~\tabularnewline
\rowcolor{even_color}DCGAN \cite{radford_etal_iclr15_unsupervised} & 6.40$\pm$0.05\tabularnewline
DFM \cite{warde2017improving} & 7.72$\pm$0.13\tabularnewline
\rowcolor{even_color}\textbf{$\model$} & \textbf{7.15$\pm$0.07}\tabularnewline
\hline 
\end{tabular}
\par\end{center}%
\end{minipage}\hfill{}%
\begin{minipage}[t]{0.58\textwidth}%
\noindent \begin{center}
\vspace{-2mm}
\includegraphics[width=1\textwidth]{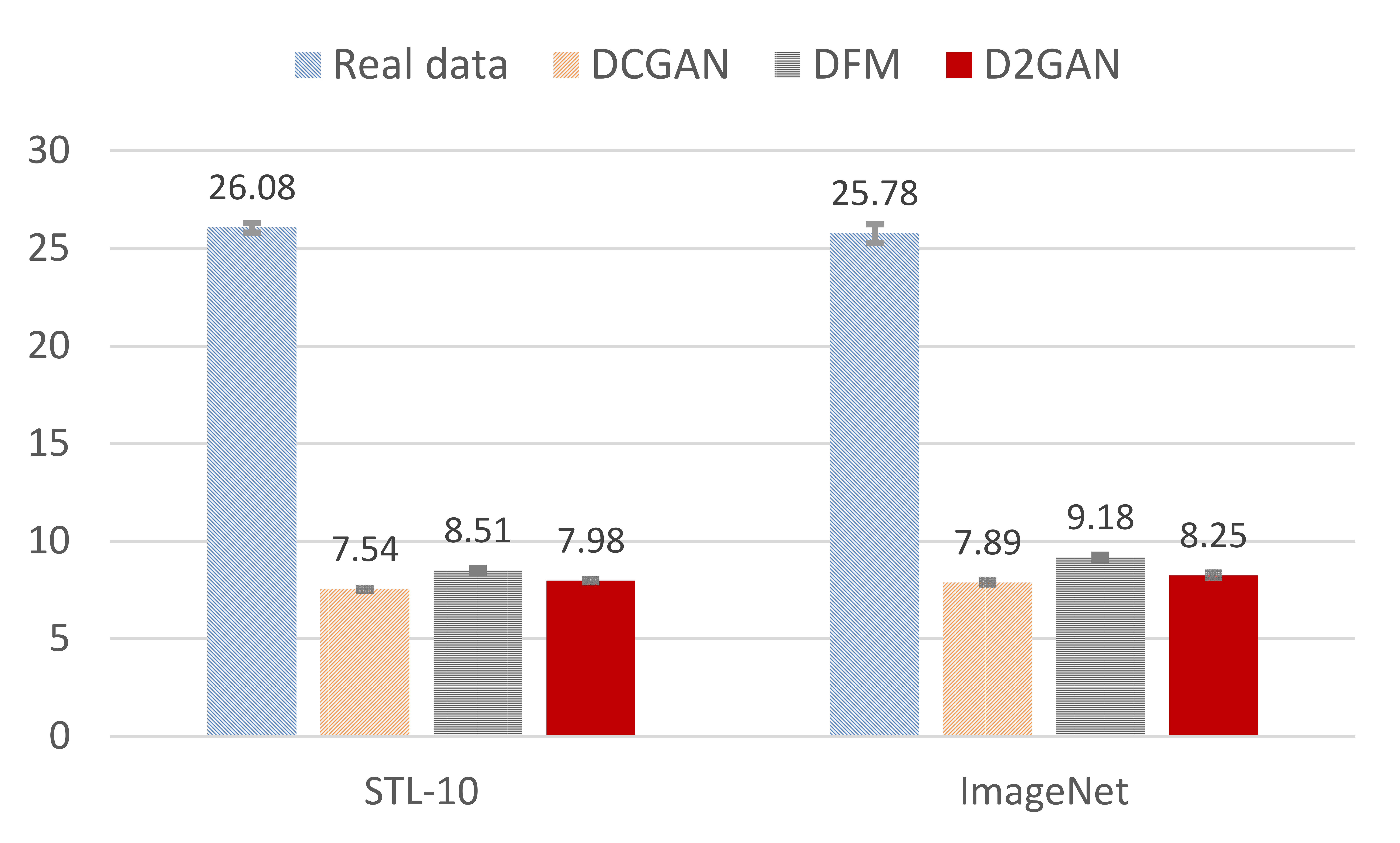}\vspace{-3mm}
\caption{Inception scores on STL-10 and ImageNet.\label{fig:exp_inception_scores_stl10_imagenet}}

\par\end{center}%
\end{minipage}\vspace{-6mm}
\end{figure}

Tab.~\ref{tab:exp_inception_scores_cifar10} and Fig.~\ref{fig:exp_inception_scores_stl10_imagenet}
show the Inception scores on CIFAR-10, STL-10 and ImageNet datasets
obtained by our model and baselines collected from recent work in
literature. It is worthy to note that we only compare with methods
trained in a completely unsupervised manner without label information.
As the result, there exist 8 baselines on CIFAR-10 whilst only DCGAN
\cite{radford_etal_iclr15_unsupervised} and denoising feature matching
(DFM) \cite{warde2017improving} are available on STL-10 and ImageNet.
We use our own TensorFlow implementation of DCGAN with the same network
architecture with our model for fair comparisons. In all 3 experiments,
the $\model$ fails to beat the DFM, but outperforms other baselines
by large margins. The lower results compared with DFM suggest that
using autoencoders for matching high-level features appears to be
an effective way to encourage the diversity. This technique is compatible
with our method, thus integrating it could be a promising avenue for
our future work.

Finally, we show several samples generated by our proposed model trained
on these three datasets in Fig.~\ref{fig:exp_samples}. Samples are
fair random draws, not cherry-picked. It can be seen that our $\model$
is able to produce visually recognizable images of cars, trucks, boats,
horses on CIFAR-10. The objects are getting harder to recognize, but
the shapes of airplanes, cars, trucks and animals still can be identified
on STL-10, and images with various backgrounds such as sky, underwater,
mountain, forest are shown on ImageNet. This confirms the diversity
of samples generated by our model.\vspace{-6mm}

\begin{figure}[h]
\noindent \begin{centering}
\subfloat[CIFAR-10.]{\noindent \centering{}\includegraphics[width=0.32\textwidth]{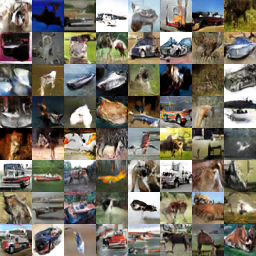}}\hfill{}\subfloat[STL-10.]{\noindent \centering{}\includegraphics[width=0.32\textwidth]{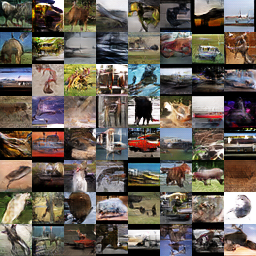}}\hfill{}\subfloat[ImageNet.]{\noindent \centering{}\includegraphics[width=0.32\textwidth]{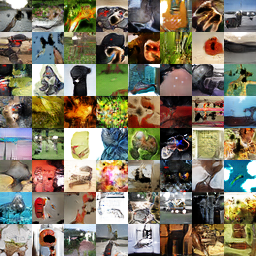}}\vspace{-2mm}

\par\end{centering}

\noindent \begin{centering}
\caption{Samples generated by our proposed $\protect\model$ trained on natural
image datasets. Due to the space limit, please refer to the supplementary
material for larger plot.\label{fig:exp_samples}}

\par\end{centering}

\noindent \centering{}
\end{figure}
\vspace{-6mm}

\section{Conclusion}

\vspace{-3mm}
To summarize, we have introduced a novel approach to combine Kullback-Leibler
(KL) and reverse KL divergences in a unified objective function of
the density estimation problem. Our idea is to exploit the complementary
statistical properties of two divergences to improve both the quality
and diversity of samples generated from the estimator. To that end,
we propose a novel framework based on generative adversarial nets
(GANs), which formulates a minimax game of three players: two discriminators
and one generator, thus termed \emph{dual discriminator GAN} ($\model$).
Given two discriminators fixed, the learning of generator moves towards
optimizing both KL and reverse KL divergences simultaneously, and
thus can help avoid mode collapse, a notorious drawback of GANs.

We have established extensive experiments to demonstrate the effectiveness
and scalability of our proposed approach using synthetic and large-scale
real-world datasets. Compared with the latest state-of-the-art baselines,
our model is more scalable, can be trained on the large-scale ImageNet
dataset, and obtains Inception scores lower than those of the combination
of denoising autoencoder and GAN (DFM), but significantly higher than
the others. Finally, we note that our method is orthogonal and could
integrate techniques in those baselines such as semi-supervised learning
\cite{salimans_etal_nips16_improved}, conditional architectures \cite{mirza2014conditional,denton_etal_nips15_deep,reed_etal_icml16_generative}
and autoencoder \cite{che2016mode,warde2017improving}.
\pagebreak{}

\bibliographystyle{plain}

\end{document}